%% file: paper.tex
\pdfoutput=1 
\documentclass{article} 
\usepackage[usenames, dvipsnames]{xcolor}
\usepackage[utf8]{inputenc}
\usepackage[english]{babel} 
\usepackage{amsmath,amsfonts}
\usepackage{amssymb}
\usepackage[textwidth=.8in]{todonotes}
\usepackage[showonlyrefs]{mathtools}
\usepackage{amsthm}
\usepackage{moresize}
\usepackage{tikz}
\usepackage[a4paper, portrait, left=1in, right=1in, top=1in, bottom=1in]{geometry} 
\usepackage[numbers,sectionbib]{natbib} 
\usepackage[hyperindex,breaklinks,bookmarks]{hyperref}
\hypersetup{bookmarksdepth=2} 
\usepackage[shortlabels]{enumitem}
\usepackage{booktabs} 
\usepackage{xparse} 
\usepackage{array}
\usepackage{booktabs}
\usepackage{pifont}
\usepackage[toc]{appendix}
\usepackage{cleveref}
\usepackage{autonum}
\usepackage{makecell}
\usepackage{enumitem}
\usepackage{filecontents}

\usetikzlibrary{patterns}

\newcommand\eqdef{\mathrel{\overset{\makebox[0pt]{\mbox{\normalfont\ssmall\sffamily def}}}{=}}}
\DeclarePairedDelimiter{\ceil}{\lceil}{\rceil}

\def\squareforqed{\leavevmode\hbox to.77778em{\hfil\vrule\vbox to.675em{\hrule width.6em\vfil\hrule}\vrule\hfil}} 
\newcommand\mae{\textsl{\ae}}  

\newtheorem{theorem}{Theorem}
\newtheorem{corollary}[theorem]{Corollary}
\newtheorem{definition}[theorem]{Definition}
\newtheorem{lemma}[theorem]{Lemma}

\theoremstyle{remark}
\newtheorem*{remark}{Remark}

\newenvironment{repeatthm}[1]{

	\smallskip \noindent
	Recall \textbf{\Cref{#1}:} \begin{em}}{\end{em} \\ \smallskip \noindent}

\date{}
\title{Performance of Bounded-Rational Agents With the Ability to Self-Modify\thanks{Supported by Grant Number 16582, Basic Algorithms Research Copenhagen
(BARC), from the VILLUM Foundation.}
}

\author{
	Jakub Tětek\\
	\texttt{j.tetek@gmail.com}\\
	BARC, Univ. of Copenhagen
	\and
	Marek Sklenka\\
	\texttt{sklenka.marek@gmail.com}\\
	University of Oxford
	\and
	Tomáš Gavenčiak\\
	\texttt{gavento@gmail.com}
}

\begin{document}
\maketitle
\begin{abstract}
Self-modification of agents embedded in complex environments is hard to avoid, whether it happens via direct means (e.g. own code modification) or indirectly (e.g. influencing the operator, exploiting bugs or the environment). It has been argued that intelligent agents have an incentive to avoid modifying their utility function so that their future instances work towards the same goals. 

Everitt et al. (2016) formally show that providing an option to self-modify is harmless for perfectly rational agents. We show that this result is no longer true for agents with bounded rationality. In such agents, self-modification may cause exponential deterioration in performance and gradual misalignment of a previously aligned agent. We investigate how the size of this effect depends on the type and magnitude of imperfections in the agent's rationality (1-4 below). We also discuss model assumptions and the wider problem and framing space.

We examine four ways in which an agent can be bounded-rational: it either \textit{(1)} doesn't always choose the optimal action, \textit{(2)} is not perfectly aligned with human values, \textit{(3)} has an inaccurate model of the environment, or \textit{(4)} uses the wrong temporal discounting factor. We show that while in the cases \textit{(2)-(4)} the misalignment caused by the agent's imperfection does not increase over time, with \textit{(1)} the misalignment may grow exponentially.

\end{abstract}

\section{Introduction}
\input{intro.tex}

\section{Preliminaries}\label{sec:preliminaries}
\input{preliminaries.tex}
\section{Definitions of bounded-rational agents}\label{sec:def_bounded_rat}
\input{definitions.tex}


\section{Exposition of the results}

\subsection{Performance of \texorpdfstring{$\epsilon$}{epsilon}-optimizers can deteriorate} \label{sec:policy_mod}
\input{policy1.tex}

\subsection{Bounded-knowledge agents are \texorpdfstring{$\epsilon$}{epsilon}-optimizers} \label{sec:optimization_bounds}
\input{bounded_eps1.tex}

\subsubsection{\texorpdfstring{$\epsilon$}{epsilon}-misaligned agents are \texorpdfstring{$\epsilon'$}{epsilon'}-optimizers}\label{sec:utility}
\input{utility1.tex}

\subsubsection{\texorpdfstring{$\epsilon$}{epsilon}-ignorant agents are \texorpdfstring{$\epsilon'$}{epsilon'}-optimizers} \label{sec:belief}
\input{belief1.tex}

\subsubsection{Impatient agents are \texorpdfstring{$\epsilon$}{epsilon}-optimizers} \label{sec:gammas}
\input{gammas1.tex}

\subsection{Combining the results} \label{sec:combining}
\input{combining1.tex}

\section{Future work} \label{sec:conclusions}
\input{conclusions.tex}

\section{Derivation of the results}

\subsection{Performance of \texorpdfstring{$\epsilon$}{epsilon}-optimizers can deteriorate}
\label{sec:policy_mod_techn}
\input{policy2.tex}

\subsection{Bounded-knowledge agents are \texorpdfstring{$\epsilon$}{epsilon}-optimizers}
\label{sec:optimization_bounds_techn}
\input{bounded_eps2.tex}

\subsubsection{\texorpdfstring{$\epsilon$}{epsilon}-misaligned agents are \texorpdfstring{$\epsilon'$}{epsilon'}-optimizers}
\label{sec:belief_techn}
\input{utility2.tex}

\subsubsection{\texorpdfstring{$\epsilon$}{epsilon}-ignorant agents are \texorpdfstring{$\epsilon'$}{epsilon'}-optimizers}
\label{sec:belief_techn}
\input{belief2.tex}

\subsubsection{Impatient agents are \texorpdfstring{$\epsilon$}{epsilon}-optimizers}
\label{sec:gammas_techn}
\input{gammas2.tex}

\subsection{Combining the results}
\label{sec:combining_techn}
\input{combining2.tex}

\section{Acknowledgements}
This work was carried out as part of the AI Safety Research Program 2019.
We are grateful to Vojtěch Kovařík, Jan Kulveit, and Gavin Leech (in alphabetical order) for their valuable inputs.

\bibliographystyle{plainnat}
\bibliography{literature}
\end{document}

%% file: intro.tex
We face the prospect of creating superhuman (or otherwise very powerful) AI systems in the future where those systems hold significant power in the real world \cite{Superintelligence,russell2019human}. Building up theoretical foundations for the study and design of such systems gives us a better chance to align them with our long-term interests. In this line of work, we study agent-like systems, i.e. systems optimizing their actions to maximize a certain utility function -- the framework behind the current state-of-the-art reinforcement learning systems and one of the major proposed models for future AI systems\footnote{Other major models include e.g. comprehensive systems of services \cite{drexler2019reframing} and "Oracle AI" or "Tool AI"~\cite{OracleAI}. However, there are concerns and ongoing research into the emergence of agency in these systems~\cite{omohundro2008basic,miller2020agi}.}.

If strong AI systems with the ability to act in the real world are ever deployed\footnote{Proposals to prevent this include e.g. boxing~\cite{Superintelligence} but as e.g. \citet{Leakproofing} argues, this may be difficult or impractical.}, it is very likely that they will have some means of deliberately manipulating their own implementation, either directly or indirectly (e.g. via manipulating the human controller, influencing the development of a future AI, exploiting their own bugs or physical limitations of the hardware, etc). While the extent of those means is unknown, even weak indirect means could be extensively exploited with sufficient knowledge, compute, modelling capabilities and time.

\smallskip
\citet{omohundro2008basic} argues that every intelligent system has a fundamental drive for goal preservation, because when the future instance of the same agent strives towards the same goal, it is more likely that the goal will be achieved. Therefore, Ohomundro argues, a rational agent should never modify into an agent optimizing different goals.

\citet{everitt} examine this question formally and arrive at the same conclusion: that the agent preserves its goals in time (as long as the agent's planning algorithm anticipates the consequences of self-modifications and uses the current utility function to evaluate different futures).\footnote{\citet{everitt}'s results hold independent of the length of the time horizon or temporal discounting (by simple utility scaling).
} However, Everitt's analysis assumes that the agent is a perfect utility maximizer (i.e. it always takes the action with the greatest expected utility), and has perfect knowledge of the environment. These assumptions are probably unattainable in any complex environment.

\smallskip
To address this, we present a theoretical analysis of a self-modifying agent with imperfect optimization ability and incomplete knowledge. We model the agent in the standard cybernetic model where the agent can be bounded-rational in two different ways. Either the agent makes suboptimal decisions (is a bounded-optimization agent) or has inaccurate knowledge. We conclude that imperfect optimization can lead to exponential deterioration of alignment through self-modification, as opposed to bounded knowledge, which does not result in future misalignment. An informal summary of the results is presented below.

\smallskip
Finally, we explicitly list and discuss the underlying assumptions that motivate the theoretical problem and analysis. In addition to clearly specifying the \emph{scope of conclusions}, the explicit problem assumptions can be used as a rough axis to map the space of viable research questions in the area; see Sections~\ref{sec:assumptions} and~\ref{sec:conclusions}.

\subsection{Summary of our results} \label{sec:summary}

The result of \citet{everitt} could be loosely interpreted to imply that agents with close to perfect rationality would either prefer not to self-modify, or would self-modify and only lose a negligible target value. 

We show that when we relax the assumption of perfect rationality, their result no longer applies. The bounded-rational agent may prefer to self-modify given the option and in doing so, become less aligned and lose a significant part of the attainable value according to its original goals.

\smallskip
We use the difference between the attainable and attained expected future value at an (arbitrarily chosen) future time point as a proxy for the degree of the agent's misalignment at that time. Specifically, for a future time $t$, we consider the value attainable from time $t$ (after the agent already ran and self-modified for $t$ time units), and we estimate the loss of value $f^t$ relative to the non-modified agent in the same environment state. Note that $f^t$ is not pre-discounted by the previous $t$ steps. See Section~\ref{sec:preliminaries} for formal definitions and Section~\ref{sec:assumptions} for motivation and discussion.

\smallskip
We consider four types of deviation from perfect rationality, see Section~\ref{sec:def_bounded_rat} for formal definitions.

\begin{itemize}
\item \emph{$\epsilon$-optimizers} make suboptimal decisions.
\item \emph{$\epsilon$-misaligned agents} have inaccurate knowledge of the human utility function.
\item \emph{$\epsilon$-ignorant agents} have inaccurate knowledge of the environment.
\item \emph{$\epsilon$-impatient agents} have inaccurate knowledge of the correct temporal discount function.
\end{itemize}

Note that for the sake of simplicity, we use a very simple model of bounded rationality where the errors are simply bounded by the error parameters $\epsilon_\bullet$; this has to be taken into account when interpreting the results. However, we suspect that the asymptotic dependence of value loss on the size of errors and time would be similar for a range of natural, realistic models of bounded rationality.

\medskip
\noindent
\textbf{Informal result statements}

\medskip
\noindent
\emph{Self-modifying $\epsilon$-optimizers} may deteriorate in future alignment and performance exponentially over time, losing exponential amount of utility compared to $\epsilon$-optimizers that do not self-modify. We show upper and tight lower bounds (by a constant) on the worst-case value loss in Theorem~\ref{thm:policy_modification}. As we decrease $\gamma$ (increase discounting), the rate at which the agent's performance deteriorates increases and the possibility of self-modification becomes a more serious problem. 

Our analysis of bounded-optimization agents is a generalization of Theorem 16 from \citet{everitt} in the sense that their result can be easily recovered by a basic measure-theoretic argument.

\medskip
\noindent
\emph{Self-modifying $\epsilon_u$-misaligned, $\epsilon_\rho$-ignorant, or $\epsilon_\gamma$-impatient perfect optimizers} can only lose the same value as non-self-modifying agents with the same irrationality bounds. This also holds for any combination of the three types of bounded knowledge. We give tight upper and lower bounds (up to a constant factor) for the worst-case performance. See Section~\ref{sec:optimization_bounds} for details.

This implies that unlike bounded-optimization agents, the performance of perfect-optimization bounded-knowledge agents does not deteriorate in time. This is because bounded-knowledge agents continue to take optimal actions with respect to their almost correct knowledge and do not self-modify in a way that would worsen their performance in their view. Therefore, the possibility of self-modification seems less dangerous in the case of bounded-knowledge agents than in the case of bounded-optimization agents.

\medskip
\noindent
A \emph{self-modifying agent with any combination of the four irrationality types} may lose value exponential in the time step $t$ when the agent optimization error parameter $\epsilon_o > 0$. We again give tight (up to a constant factor) lower bounds on the worst-case performance of such agents. See Section~\ref{sec:combining} for details.

\smallskip
Our results do not imply that every such agent will actually perform this poorly but the prospect of exponential deterioration is worrying in the long-term, even if it happens at a much slower speed than suggested by our results. We focus on worst-case analysis because it tells us whether we can have formal guarantees of the agent's behaviour -- a highly desirable property for powerful real-world autonomous systems, including a prospective AGI (artificial general intelligence) or otherwise strong AIs.

\medskip
\noindent
\textbf{Overview of formal results.} Here we summarize how much value the different types of bounded-rational agents may lose via misalignment. Note that the maximal attainable discounted value is at most $\frac{1}{1-\gamma}$ and the losses should be considered relative to that, or to the maximum attainable value in concrete scenarios. Otherwise, the values for different value of $\gamma$ are incomparable.
In all cases, the worst-case lower and upper bounds are tight up to a constant.

\medskip
\noindent
\emph{$\epsilon$-optimizer agents} -- bounded optimization, after $t$ steps of possible self-modification (Theorem~\ref{thm:policy_modification})
$$f^t_\text{opt}(\epsilon, \gamma) = \min(\frac{\epsilon}{\gamma^{t-1}}, \frac{1}{1-\gamma})$$ 

\medskip
\noindent
\emph{$\epsilon$-misaligned agents} -- inaccurate utility (Theorem~\ref{thm:inaccurate_utility})
$$f_\text{util}(\epsilon, \gamma) = \frac{2\epsilon}{1-\gamma}$$

\medskip
\noindent
\emph{$\epsilon$-ignorant agents} -- inaccurate belief (Theorem~\ref{thm:inaccurate_belief})
$$f_\text{bel}(\epsilon, \gamma) = \frac{2}{1-\gamma} - \frac{2}{1-\gamma(1-\epsilon)}$$

\smallskip
\noindent
\emph{$\epsilon$-impatient agents} -- inaccurate discounting (Theorem~\ref{thm:imprecise_discount}) Here $\gamma^*$ is the correct discount factor and $\gamma$ is the agent's incorrect discount factor.
$$f_\text{disc}(\gamma, \gamma^*) \approx \frac{2 {\gamma^*}^\frac{1}{\lg \gamma}-1}{1-\gamma^*}$$
    
\section{Assumptions and rationale}\label{sec:assumptions}

Both the statement of the problem and its relevance to AI alignment rest on a set of assumptions listed below. While this list is non-exhaustive, we try to cover the main implicit and explicit choices in our framing, and the space of alternatives. This is largely in hope of eventually finding a better, more robust theoretical framework for solving agent self-modification within the context of AI alignment, but even further negative results in the space would inform our intuitions on what aspects of self-modification make the problem harder.

\smallskip
We propose consideration of various assumptions as a framework for thinking about prospective \emph{realistic agent models that admit formal guarantees}. We invite further research and generalizations in this area, one high-level goal being to map a part of the space of agent models and assumptions that do or do not permit guarantees, eventually finding agent models that do come with meaningful guarantees. Further negative results would inform our intuitions on what aspects of the problems make it harder.

\begin{enumerate}[(i)]
    \item \emph{Bounded rationality model.} In the models of $\epsilon$-bounded-rational agents defined in Section~\ref{sec:bounded_optimization}, $\epsilon$ is generally an upper bound on the size of the optimization or knowledge error. One interpretation of our results is that value drift can happen even if the error is bounded at every step. One could argue that a more realistic scenario would assume some distribution of the size of the errors, assuming larger errors less likely or less frequent; see discussion below and in Section~\ref{sec:conclusions}.
    \item \emph{Unlimited self-modification ability.} We assume the agent is able to perform any self-modification at any time. This models the worst-case scenario when compared to a limited but still perfectly controlled self-modification. However, embedded (non-dualistic) agents in complex environments may chieve almost-unlimited self-modification from a limited ability, e.g. over a longer time span; see e.g. \cite{EmbeddedAgency}. We model the agent's self-modifications as orthogonal to actions in the environment.
    \item \emph{Modification-independence.} We assume that the agent's utility function does not explicitly reward or punish self-modifications. We also assume that self-modifications do not have any direct effect on the environment. This is captured by \Cref{def:mod_indep}.
    \item \emph{No corrigibility mechanisms.} We do not consider systems that would allow human operators to correct the system's goals, knowledge or behaviour. The problem of robust strong AI corrigibility is far from solved today and this paper can be read as a further argument for substantially more research in this direction.
    \item \emph{Worst-case analysis and bound tightness.} We focus on worst-case performance guarantees in abstracted models rather than e.g. full distributional analysis, and we show that our worst-case bounds are attainable (up to constant factors) under certain agent behaviour. Note this approach may turn out as too pessimistic or even impossible in some settings (e.g. quantum physics).
    \item \emph{Bounded value attainable per time unit.} We assume the agent obtains instantaneous utility between 0 and 1 at each time step. This is not an arbitrary choice: A constant bound on instantaneous value can be normalized to this interval. Instantaneous values bounded by a function of time $U(t)<\mu^t$ can be pre-discounted when $\gamma\mu<1$, and generally lead to infinite future values otherwise, which we disallow here to avoid foundational problems.
    \item \emph{Temporal value discounting.} We assume the agent employs some form of temporal value discounting. This could be motivated by technical or algorithmic limitations, increasing uncertainty about the future, or to avoid issues with incomparable infinite values of considered futures (see \citet{Bostrom2011} for a discussion of infinite ethics). Discounting, however, contrasts with the long-termist view; see the discussion below.
    \item \emph{Exponential discounting.} Our model assumes the agent discounts future utility exponentially, a standard assumption in artificial intelligence and the only time-invariant discounting schema \cite{strotz1955myopia} leading to consistent preferences over time.
    \item \emph{Unbounded temporal horizons.} Our analysis focuses on the long-term behaviour of the agent, in particular stability and performance from the perspective of future stakeholders (sharing the original utility function). Note that our results also to some extent apply to finite-horizon but long-running systems.

    Temporal discounting contrasts with the long-termist view: Why not model non-discounted future utility directly? Noting the motivations we mention in (vii), we agree that models of future value aggregation other than discounting would be generally better suited for long-term objectives. However, this seems to be a difficult task, as such models are neither well developed nor currently used in open-ended AI algorithms (with the obvious exception of a finite time horizon, which we propose to explore in Section~\ref{sec:conclusions}).
    
    We therefore propose a direct interpretation of our results: \emph{Assuming we implement agents that are $\epsilon$-optimizers with discounting, they may become exponentially less aligned over time. This is not the case with perfect optimizers with imperfect knowledge and discounting.}
    \item \emph{Dualistic setting.} We assume a dualistic agent and allow self-modification through special actions. This allows us to formally model one aspect of embedded agency -- at least until there are sufficient theoretical foundations of embedded agency.

    Note that in the embedded (non-dualistic) agent setting, it is not formally clear -- or possibly even definable -- what constitutes a self-modification, since there is no clear conceptual boundary between the agent and the environment, as discussed by \citet{EmbeddedAgency}.
\end{enumerate}

\paragraph{Assumption categories and the problem space.}

Each assumption identifies a subspace of research questions we would obtain by varying the relevant choices. These subspaces vary from very technical (e.g. concrete rationality model) to foundational (e.g. finite values and dualistic agent models). Along this axis, the assumptions and choices point to different kinds of prospective problems; we briefly describe three such categories and their prospects. See Section~\ref{sec:conclusions} for concrete proposals of future work.

\medskip
\noindent
\emph{Technical choices:} A concrete model of bounded rationality, unlimited self-modification model, and modification-independence. These are likely important for short and medium time-frames, where even eventually-diverging guarantees are useful.

We believe that many models within some realistic and sufficiently strong model classes would lead to qualitatively equivalent results in long time horizons; e.g. the agent divergence would be asymptotically exponential without external corrigibility, embedded agents in sufficiently complex environments would be able to self-modify arbitrarily over a long time (see discussion above) etc. However, these intuitions call for further verification.

\medskip
\noindent
\emph{Problem components:} No corrigibility mechanisms, unbounded time horizon, time-invariant temporal discounting, focus on the worst-case guarantees. For those, there are interesting alternatives that may yield more optimistic results. In particular, it would be valuable to explore formal models of corrigibility, perform a full probabilistic analysis of agent development, and develop long-term non-discounted finite-time settings.

\medskip
\noindent
\emph{Foundational assumptions:} Dualistic agent model and finite value of the future. Those are a standard in the area, but alternative settings may open up important and fruitful model classes and technical choices that capture currently pre-paradigmatic aspects (e.g. theory of embedded agency and non-dualistic agents).

%

%% file: preliminaries.tex
	In this section, we explain our model of a self-modifying agent, which is borrowed from \citet{everitt}. We will extend this model to include bounded rationality in \Cref{sec:def_bounded_rat}.

	We use a modified version of the \textit{standard cybernetic model}. In this model, an agent interacts with the environment in discrete time steps. At each time step $t$, the agent performs an \textit{action} $a_t$ from a finite set $\mathcal{A}$ and the environment responds with a \textit{perception} $e_t$ from a finite set $\mathcal{E}$. An \textit{action-perception pair} $\mae_t$ is an action concatenated with a perception. A \textit{history} is a sequence of action-perception pairs $\mae_1\mae_2...\mae_t$. We will often abbreviate such sequences to $\mae_{< t}=\mae_1...\mae_{t-1}$ or $\mae_{n:m}=\mae_n...\mae_m$. A \textit{complete} history $\mae_{1:\infty}$ is a history containing information about all the time steps.
%
%

    An agent can be described by its policy $\pi$. The policy\footnote{For a set $S$, $S^*$ denotes the set of finite sequences of elements from $S$} $\pi \colon (\mathcal{A \times E})^* \to \mathcal{A}$ is used to determine the agent's next action from the history at time $t$. We consider (bounded-rational) utility maximizers, where the policy is (partially) determined by the instantaneous utility function $u$, belief $\rho$ and discount factor $\gamma$. We sometimes use the notation $\kappa = (u, \rho, \gamma)$, where $\kappa$ is called the agent's \emph{knowledge}. The utility function $\tilde{u} \colon (\mathcal{A \times E})^\infty \to \mathbb{R}$ describes how much the agent prefers the complete history $\mae_{1:\infty}$ compared to other complete histories. We will assume that the total utility is a discounted sum of instantaneous utilities given by the instantaneous utility function $u \colon (\mathcal{A \times E})^* \to [0,1]$. Formally, $ \tilde{u}(\mae_{1:\infty}) = \sum ^{\infty }_{t=1}\gamma^{t-1} u(\mae_{\leq t})$. The discount factor $\gamma$ describes how much the agent prefers immediate reward compared to the same reward at a later time. Smaller $\gamma$ means heavier discounting of the future and stronger preference for immediate reward. Note that the maximum achievable utility is $\frac{1}{1-\gamma}$, which happens when $u(\mae_t)=1$ at each step. Also note that instantaneous utility depends not only on the latest perception but can also depend on all previous perceptions and actions.
	
	In addition to all this, an agent has a belief $\rho \colon (\mathcal{A \times E})^* \times \mathcal{A} \to \Delta \bar{\mathcal{E}}$ where $\Delta \bar{\mathcal{E}}$ is the set of full-support probability distributions over $\mathcal{E}$. This is a function which maps any history ending with an action onto a probability distribution over the next perceptions. Intuitively speaking, the belief describes what the agent expects to see after it performs an action given a certain history. A belief together with a policy induce a measure on the set $(\mathcal{A \times E})^*$ using $P(e_t \mid \mae_{< t}a_t) = \rho(e_t \mid \mae_{< t}a_t)$  and $P(a_t \mid \mae_{< t}) = 1$ if $\pi(\mae_{<t})=a_t$ and 0 otherwise. Intuitively speaking, this probability measure captures probabilities assigned by the agent to possible futures.
	
	Following the reinforcement learning literature, we define the \textit{value function} $V\colon (\mathcal{A \times E})^* \to \mathbb{R}$ as the expected future discounted utility:
	\[ V^\pi(\mae_{<t})=\mathbb{E}[\ \sum_{t'=t}^\infty \gamma^{t'-t}u(\mae_{<t'}) \ ]\]
	The expectation value on the right is calculated with respect to belief $\rho$ and assuming the agent will follow the policy $\pi$. Intuitively, the value function describes how promising the future seems. When the value \(V^\pi(\mae_{<t})\) of a history is high, it means we can expect an agent with policy $\pi$ to collect a lot of utility in the future starting from this history. Note that when calculating V-values, instantaneous utilities are multiplied by \(\gamma^{t'-t}\)  rather than  \(\gamma^{t'}\). This means that V-values can remain high throughout the whole history and are not affected by discounting. We define the Q-value of an action as the expected future discounted utility after taking that action:
	\[ Q^\pi(\mae_{<t} a_t)=\mathbb{E}[u(\mae_{1:t}) + V^{\pi}(\mae_{1:t})] \]
	where the expectation is over the next perception drawn from the belief (note that belief is a probability distribution)

	The Q-value measures how good an action is given that the agent will later follow policy $\pi$. A policy $\pi^*$ is an \textit{optimal policy} when $V^{\pi^*}(\mae_{<t}) = \sup_{\pi} V^{\pi}(\mae_{<t})$ for all histories $\mae_{<t}$ (such a policy always exists, as shown in \cite{Lattimore2014}). 	
	\subsection{Self-modification model}
	In this section, we extend the formalism above to include the possibility of self-modification. Since we are interested in the worst-case scenario, we assume the agent has unlimited self-modification ability. Worst-case results derived for such an agent will also hold for agents with limited ability to self-modify.
	
	\begin{definition}
		A policy self-modification model is defined as a quadruple $(\breve{\mathcal{A}}, \mathcal{E, P}, \iota)$ where $\breve{\mathcal{A}}$ is the set of \textit{world actions}, $\mathcal{E}$ is the set of \textit{perceptions}, $\mathcal{P}$ is a non-empty set of \textit{names} and $\iota$ is a map from $\mathcal{P}$  to the set of all policies $\Pi$. 
	\end{definition}
	At every time step, the agent chooses an action $a_t = (\breve{a}_t, p_{t+1}
	)$ from the set $\mathcal{A} = (\breve{\mathcal{A}} \times \mathcal{P})$. The first part $\breve{a}_t$ describes what the agent ``actually does in the world'' while the second part chooses the policy $\pi_{t+1}=\iota(p_{t+1})$ for the next time step. We will also use the notation  $a_t = (\breve{a}_t, \pi_{t+1}
	)$, keeping in mind that only policies with names may be chosen. This new policy is used in the next step to pick the action $a_{t+1}=\pi_{t+1}(\mae_{1:t})$. Note that only policies with names can be chosen and that $\mathcal{P} = \Pi$ is not a possibility because it entails a contradiction: 
	$ |\Pi| = |(\breve{\mathcal{A}} \times \mathcal{E}\times \Pi)|^{|(\breve{\mathcal{A}} \times \mathcal{E}\times \Pi)^*|} > 2^{|\Pi|}>|\Pi|$. A history can now be written as:
	\[\mae_{1:t} = a_1 e_1 a_2 e_2 ... a_t e_t = \breve{a}_1 \pi_2 e_1 \breve{a}_2 \pi_3 e_2 ... \breve{a}_t \pi_{t+1} e_t \]
	The subscripts for policies are one time step ahead because the policy chosen at time $t$ is used to pick an action at time $t+1$. The subscript denotes at which time step the policy is used. Policy $\pi_t$ is used to choose the action $a_t = (\breve{a_t}, \pi_{t+1})$. No policy modification happens when $a_t = (\breve{a_t}, \pi_t)$.

	In the previous section, we used these rules to calculate the probability of any finite history: $P(e_t \mid \mae_{< t}a_t) = \rho(e_t \mid \mae_{< t}a_t)$  and $P(a_t \mid \mae_{< t}) = 1$ if $\pi(\mae_{<t})=a_t$ and 0 otherwise. However, the second rule doesn't take into consideration that the agent's policy is changing. Therefore, to account for self modification, we need to modify the second rule into ``$P(a_t\mid \mae_{< t}) = 1$ if $\pi_t(\mae_{<t})=a_t$ and zero otherwise''.  To evaluate the V and Q-functions for self-modifying agents, we need to use probabilities of complete histories calculated in this way.
	
	\begin{definition} \label{def:mod_indep}
		Let $\breve{\mae}_{1:t}$ denote the history $\mae_{1:t}$ with information about self-modification removed so that $\breve{\mae}_{1:t}=\breve{a}_1 e_1\breve{a}_2 e_2 ... \breve{a}_{t} e_t$. A function $ f \colon (\mathcal{A \times E})^* \to (anything)$ is \textit{modification-independent} if $\breve{\mae}_{1:t}=\breve{\mae}_{1:t}^{'}$ implies that $f(\mae_{1:t})=f(\mae_{1:t}^{'})$.
	\end{definition}

\paragraph{Modification-independence assumption:}In the rest of the paper, we will assume that the agent's belief and utility function as well as the correct belief are modification-independent.

%% file: definitions.tex
We now extend the model from \citet{everitt} by defining two types of bounded-rational agents which we will be using throughout the paper: bounded-optimization agents (described in \Cref{sec:bounded_optimization}) and bounded-knowledge agents (described in \Cref{sec:combining}). Bounded-knowledge agents can be subdivided further into misaligned, ignorant and impatient agents.
\begin{figure}[h] 
\includegraphics[width=13.7cm]{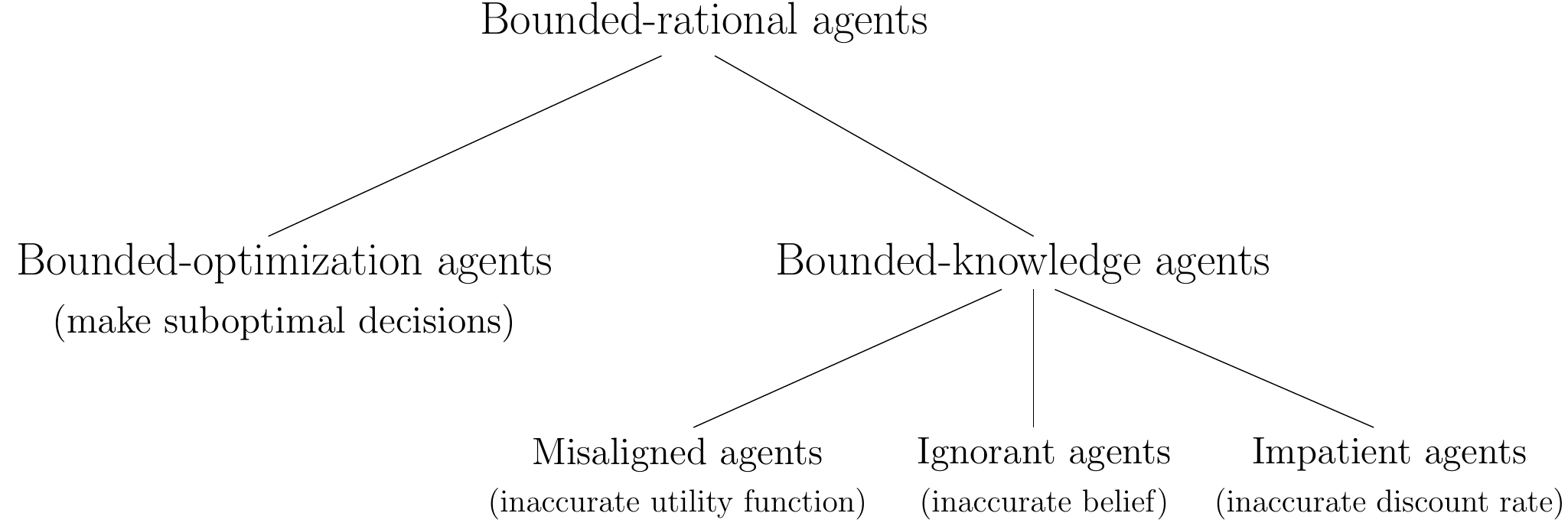}
\centering
\end{figure} 

\subsection{Bounded-optimization agents} \label{sec:bounded_optimization}
We introduce the notion of \textit{$\epsilon$-optimizers}. Intuitively speaking, the expected future discounted utility gained by an $\epsilon$-optimizer is no more than $\epsilon$ lower than the optimal one in any situation they could get into (that is, for any history).
\begin{definition}
We say that agent $A$ is an $\epsilon$-optimizer for history $\mae_{<t}$ if it holds that 
\begin{equation} \label{eq:def_eps_opt}
Q(\mae_{<t}\pi(\mae_{<t})) \geq \sup_{\pi'} Q(\mae_{<t}\pi'(\mae_{<t})) - \epsilon
\end{equation}
\end{definition}

When the utility function, belief and discount factor is obvious from the context (or unimportant), we also speak of policy being $\epsilon$-optimizing (with respect to the utility function and belief), meaning that the corresponding agent is an $\epsilon$-optimizer.


\subsection{Bounded-knowledge agents} 
We consider agents with inaccurate knowledge of the correct utility function (\Cref{def:inaccurate_utility}), inaccurate knowledge of the world (\Cref{def:inaccurate_belief_abs,def:inaccurate_belief_rel}), and inaccurate knowledge of the correct discount factor (how much future reward is worth compared to reward in the present).

\subsubsection{Misaligned agents} 
We define $\epsilon$-misaligned agents as agents whose utility function $u$ has absolute error $\epsilon$ with respect to the correct utility function $u^*$. 
\begin{definition} \label{def:inaccurate_utility}
We say that the instantaneous utility function $u$ has absolute error $\epsilon$ with respect to the correct utility function $u^*$ if
\[
\sup_{t\in \mathbb{N},\mae_{<t}} |u(\mae_{<t}) - u^*(\mae_{<t})| = \epsilon 
\]
\end{definition}

\subsubsection{Ignorant agents} 
We define $\epsilon$-ignorant agents as agents whose belief $\rho$ has error (absolute or relative depending on the context) $\epsilon$ with respect to the correct belief $\rho^*$. 

For belief, we define both relative and absolute error. This is in contrast with utility, for which this does not make sense in our setting. This is because when one speaks of relative utility, one usually compares it to some default action of ``doing nothing" which we do not have.

\begin{definition} \label{def:inaccurate_belief_abs}
We say that a belief $\rho$ has absolute error $\epsilon$ with respect to the correct belief $\rho^*$ if for any $t \in \mathbb{N}$, history $\mae_{<t}$ and action $a$,
\begin{align} 
\|\rho(\mae_{<t}a) - \rho^*(\mae_{<t}a)\|_{TV} \leq \epsilon\tag{$\star$}
\end{align}
where $\|\bullet\|_{TV}$ is the total variational distance. 
\end{definition}
Recall that (on discrete measure spaces where all subsets are measurable) for two distributions (formally two probability measures) $\mu$, $\nu$ on $\mathcal{E}$, the total variational distance is defined as
\[
\|\mu - \nu\|_{TV} = \sup_{E \subseteq \mathcal{E}} |\mu(E) - \nu(E)|
\]
\begin{definition} \label{def:inaccurate_belief_rel}
We say a belief $\rho$ has relative error $\epsilon$ with respect to the correct belief $\rho^*$ if for any $t \in \mathbb{N}$, any history $\mae_{<t}$, action $a$, and percept $e$,
\begin{equation}
\frac{1}{1+\epsilon} \leq \frac{\rho(e|\mae_{<t}a)}{\rho^*(e|\mae_{<t}a)} \leq 1+\epsilon
\end{equation}
\end{definition}

\paragraph{Impatient agents}
We define impatient agents as agents whose discount factor $\gamma$ is smaller than the correct discount factor $\gamma^*$. This means they have a stronger preference for immediate reward compared to the same reward in the future.

%% file: policy1.tex
In their paper, \citet{everitt} show that, for modification-independent belief and utility function, if we start with a perfect expected utility maximizer and at any time replace the current policy by the initial policy, the expected discounted utility stays the same. Therefore, later policies cannot be worse than the initial policy and no deterioration happens. We show that in the case of $\epsilon$-optimizers, such a replacement can never decrease the expected discounted utility by more than $\epsilon$ (Inequality \eqref{eq:1lower_bound_policy_mod}) but can increase it more, meaning that agent's behaviour can deteriorate with time (Inequality \eqref{eq:1upper_bound_policy_mod}). Specifically, it can deteriorate at an exponential rate, until its actions become arbitrarily bad -- that is, until the expected future utility lost is the maximum achievable utility (which is $\frac{1}{1-\gamma}$).
\begin{theorem} \label{thm:policy_modification}
Let $\rho$ and $u$ be modification-independent. Consider a self-modifying agent which is an $\epsilon$-optimizer for the empty history. Then, for every $t \geq 1$, \begin{align}
E_{\mae_{<t}} [Q\left(\mae_{<t} \pi_{t}\left(\mae_{<t}\right)\right)] \geq E_{\mae_{<t}}[Q\left(\mae_{<t} \pi_{1}\left(\mae_{<t}\right)\right)] - \min(\frac{\epsilon}{\gamma^{t-1}}, \frac{1}{1-\gamma}) \label{eq:1upper_bound_policy_mod} 
\end{align}
where the expectation is with respect to $\mae_{<t}$ such that the perceptions are distributed according to the belief and the actions are given by $a_{i}=\pi_{i}\left(\mae_{<i}\right)$.

Moreover, for all histories $\mae_{<t}$ given by $a_{i}=\pi_{i}\left(\mae_{<i}\right)$ for which the agent is an $\epsilon$-optimizer it holds that
\begin{gather}
Q\left(\mae_{<t} \pi_{1}\left(\mae_{<t}\right)\right) + \epsilon \geq Q\left(\mae_{<t} \pi_{t}\left(\mae_{<t}\right)\right) \label{eq:1lower_bound_policy_mod}\\
\end{gather}
Equality in Inequality \eqref{eq:1upper_bound_policy_mod} can be achieved up to a factor of at most $\gamma$.
\end{theorem}

The expectation in inequality \eqref{eq:1upper_bound_policy_mod} is necessary as can be demonstrated by the following example. Consider an environment in which the first perception is $\alpha$ with probability $\epsilon (1-\gamma)$ and $\beta$ otherwise. Regardless of the first perception, the utility in the future is always $1$ if the action following this perception is $a$ and $0$ if $b$. An $\epsilon$-optimizing agent which performs action $a$ may choose to self-modify to an agent which performs action $a$ after perception $\alpha$ and $b$ after perception $\beta$, thus losing $\frac{\gamma}{1-\gamma}$ utility in the case of perception $\beta$, regardless of how small $\epsilon$ is.

\smallskip \noindent
Setting $\epsilon=0$ allows us to easily recover Theorem 16 from \cite{everitt}, showing that self-modifications do not impact expected discounted utility gained by perfectly rational agents. This proof is also considerably simpler than the one in the original paper. For the proof, see \Cref{cor:recover_everitt}. 

If we only care about future discounted utility, this deterioration in performance doesn't need to concern us because it only happens at future times when utility is heavily discounted. From the definition of an $\epsilon$-optimizer, the maximum utility lost is indeed only $\epsilon$. On the other hand, if we care about long-term performance of the agent and have only introduced the discount factor for instrumental reasons (as would likely be the case), the possibility of self-modification becomes a serious problem. The discount factor might be introduced because optimizing the long-term future might be computationally intractable.

%% file: bounded_eps1.tex
In this section, we discuss perfect utility maximizers with bounded knowledge and show performance guarantees for such agents. In \Cref{sec:combining}, we combine these results, show how to relax the assumption of perfect optimization and, most importantly, show how the performance of a bounded-rational agent differs between the cases with and without self-modification.

In \Cref{lem:optimization_bound}, we show that if the agent's estimate of the expected discounted utility is at most $\epsilon$ away from the true value, the agent will be a $2\epsilon$-optimizer. In \crefrange{sec:utility}{sec:gammas}, we show bounds on how inaccurate the agent's estimate of expected discounted utility can be, thus proving bounds on optimization. In \Cref{sec:gammas}, we proceed differently: we formulate the worst case as a solution of an optimization problem which we then solve analytically.

\begin{lemma} \label{lem:optimization_bound}
Let A be a (possibly self-modifying) perfect expected utility maximizer with knowledge $\kappa = (u,\rho, \gamma)$. Let $\kappa^* = (u^*,\rho^*, \gamma^*)$ be the correct knowledge. Assume that
\begin{equation} \label{eq:bound_assumption}
|V^\pi_\kappa(\mae_{<t}) - V^\pi_{\kappa^*}(\mae_{<t})| \leq \epsilon
\end{equation}
for all policies $\pi$ and histories $\mae_{<t}$. Then, A is a $2\epsilon$-optimizer with respect to $\kappa^*$.
\end{lemma}

%% file: utility1.tex
We now consider agents with an inaccurate utility function and derive bounds on $\epsilon'$ such that the misaligned agent is an $\epsilon'$-optimizer. 

\begin{theorem} \label{thm:inaccurate_utility}
Let $A$ be a perfect utility maximizer with utility function $u$ and error $\epsilon$ with respect to the correct utility function $u^*$. Then it is a $\frac{2\epsilon}{1-\gamma}$-optimizer with respect to $u^*$. Moreover, this bound is tight.
\end{theorem}

In the random-error case when for any $\mae_{<t}$, we randomly choose $u(\mae_{<t})$ from the set $\{\max(0,u^*(\mae_{<t})-\epsilon), \min(1,u^*(\mae_{<t})+\epsilon)\}$, we give a simple lower bound that is only a factor $4$ away from the upper bound.
Consider an environment with only one perception $1$ and actions $\{0,1\}$ and $u^*(1|\mae_{<t}0) = 1-2\epsilon$ and $u^*(1|\mae_{<t}1) = 1$. With probability $1/4$, 
it holds that $u(1|\mae_{<t}1) = u(1|\mae_{<t}0)$, in which case the agent may take the suboptimal action $0$, thus losing $2\epsilon$ in instantaneous utility. At every step, it therefore loses $\epsilon/2$ instantaneous expected utility. In total, it then loses in expectation $\frac{\epsilon}{2(1-\gamma)}$. We have thus proved the following:

\begin{theorem}
Let $A$ be a perfect utility maximizer whose utility function $u$ is such that for any $e_t,\mae_{<t}a_t$, the value $u(e_t|\mae_{<t}a_t)$ is chosen independently and uniformly from the set $\{\max(0,u^*(e_t|\mae_{<t}a_t)-\epsilon), \min(1,u^*(e_t|\mae_{<t}a_t)+\epsilon)\}$. Then, the amount of utility lost is in expectation
\[
\frac{\epsilon}{2(1-\gamma)}
\]
\end{theorem}

%% file: belief1.tex
In this section, we discuss agents with inaccurate belief. \Cref{thm:inaccurate_belief} gives bounds on the utility lost as a result of the agent having an inaccurate belief. We give an upper bound in terms of the (weaker) absolute error and lower bounds in terms of both absolute and relative error, showing that the upper bound is tight up to factors of $2$ and $4$ for absolute and relative error.

\begin{theorem} \label{thm:inaccurate_belief}
Let $A$ be a perfect expected utility maximizer whose belief $\rho$ has absolute error $\epsilon$ with respect to the correct belief $\rho^*$. Then it is a $(\frac{2}{1-\gamma} - \frac{2}{1- \gamma(1-\epsilon)})$-optimizer with respect to $\rho^*$ and this bound is tight up to a factor of $2$. Moreover, if $\epsilon$ is the relative error, this bound is tight up to a factor of 4.
\end{theorem}

So far, we have considered the worst-case scenario. In the next theorem, we show that in the case of both absolute and relative error, the upper bound is tight up to a constant factor even in the case when the error at each timestep is randomly chosen from the set $\{-\epsilon, \epsilon\}$ (that is, $\rho(e_t|\mae_{<t}a_t)$ is chosen uniformly from the set $\{\max(0,\rho^*(e_t|\mae_{<t}a_t)-\epsilon), \min(1,\rho^*(e_t|\mae_{<t}a_t)+\epsilon)\}$ in the case of absolute error and $\{\frac{\rho^*(e_t|\mae_{<t}a_t)}{1+\epsilon}, \min(1,(1+\epsilon)\rho^*(e_t|\mae_{<t}a_t))\}$ in the case of relative error), independently of other timesteps.

\begin{theorem} \label{thm:avg_case_belief}
Let $A$ be a perfect expected utility maximizer whose belief $\rho$ is such that for any $e_t,\mae_{<t}a_t$, the value $\rho(e_t|\mae_{<t}a_t)$ is independently for any argument chosen uniformly from the set $\{\max(0,\rho^*(e_t|\mae_{<t}a_t)-\epsilon), \min(1,\rho^*(e_t|\mae_{<t}a_t)+\epsilon)\}$ in the case of absolute error and $\{\frac{\rho^*(e_t|\mae_{<t}a_t)}{1+\epsilon}, \min(1,(1+\epsilon)\rho^*(e_t|\mae_{<t}a_t))\}$ in the case of relative error.

Then, in expectation, the amount of expected discounted utility lost is respectively
\begin{gather}
\frac{1}{1-\gamma} - \frac{1}{1- \gamma(1-\epsilon/8)}\\
\frac{1}{1-\gamma} - \frac{1}{1- \gamma(1-\epsilon/16)}
\end{gather}

and this is equal to the upper bound for non-random error up to a factor of $16$ and $32$, respectively.
\end{theorem}

%% file: gammas1.tex
In this section, we discuss the case when an agent has an incorrect discount factor and give a bound on the performance of this agent with respect to the correct discount factor. We only consider the case when the agent discounts faster than the correct discount rate -- we deem this to be the interesting case as, generally speaking, while optimizing in the long-term might be desirable, it is difficult to achieve, so the agent is likely to optimize in shorter term than desired. Bounds for the other case can be derived by the same method. To simplify the bounds, we define $k = \lceil - \frac{1}{\lg \gamma}\rceil$.

\begin{theorem}\label{thm:imprecise_discount}
Let $\pi_{\gamma}$ and $\pi_{\gamma^*}$ be perfect expected utility maximizers with respect to discount factors $\gamma$ and $\gamma^*$ respectively for some $\gamma \leq \gamma^*$, either with or without the ability to self-modify. Let $u$,$\rho$ be their utility function and belief. Then 
\[ 
\left|V^{\pi_{\gamma^*}}\left(\mae_{<t}\right)-V^{\pi_{\gamma}}\left(\mae_{<t}\right)\right| \leq \frac{{\gamma^*}^{k} + {\gamma^*}^{k-1} - 1}{1-\gamma^*} - {\gamma^*}^{k-1}\frac{\gamma^{k} + \gamma^{k-1} - 1}{\gamma^{k-1}(1-\gamma)}
\]
\end{theorem}

\smallskip \noindent
For $\gamma \rightarrow 1$, it holds that $\lceil - \frac{1}{\lg \gamma} \rceil \sim - \frac{1}{\lg \gamma}$. This enables us to simplify the previous result to get a good approximation for when $\gamma$ is close to $1$:
\[
\frac{{\gamma^*}^{k} + {\gamma^*}^{k-1} - 1}{1-\gamma^*} - {\gamma^*}^{k-1}\frac{\gamma^{k} + \gamma^{k-1} - 1}{\gamma^{k-1}(1-\gamma)} \approx \frac{2 {\gamma^*}^\frac{1}{\lg \gamma}-1}{1-\gamma^*}
\]

%% file: combining1.tex
In this section we combine the results from \cref{sec:policy_mod,sec:optimization_bounds} and present a bound on the utility lost by an agent which is misaligned, ignorant, impatient and has bounded optimization, all at the same time. It is an interesting feature of this bound that the worst-case performance guarantee can in some cases be improved by adjusting its discount rate.

Recall that the functions $f_\bullet$ in the following theorem have been defined in \Cref{sec:summary}.
\begin{theorem} \label{thm:combining}
Let $A$ be an $\epsilon_o$-optimizer for the empty history with either (1) the ability to self-modify and modification-independent utility function and belief, or (2) without the ability to self-modify and with a possibly modification-dependent utility function and belief. Let $\gamma$ be the agent's discount rate, $\epsilon_u$ the error in its utility function wrt. the correct utility function $u^*$ and $\epsilon_\rho$ its absolute error in belief with respect to the correct belief $\rho^*$. Then at timestep $t$:
\begin{enumerate}[(1)]
\item If we let $\epsilon'$ be the smallest possible number such that $A$ at time $t$ is an $\epsilon'$-optimizer, then $E_{\mae_{<t}}[\epsilon'] \leq f_\text{opt}(\epsilon_o, \gamma) + f_\text{util}(\epsilon_u, \gamma) + f_\text{bel}(\epsilon_\rho, \gamma) + f_\text{disc}(\gamma, \gamma^*)$ where the expectation is over histories where perceptions are distributed according to $\rho^*$ and actions are given by the agent's policy. Moreover, if $\epsilon_o = 0$, then $\epsilon' \leq f_\text{opt}(\epsilon_o, \gamma) + f_\text{util}(\epsilon_u, \gamma) + f_\text{bel}(\epsilon_\rho, \gamma) + f_\text{disc}(\gamma, \gamma^*)$ almost certainly.
\item $A$ will be an $\epsilon'$-optimizer, with respect to the correct discount rate $\gamma^*$, where $\epsilon' \leq \epsilon_o + f_\text{util}(\epsilon_u, \gamma) + f_\text{bel}(\epsilon_\rho, \gamma) + f_\text{disc}(\gamma, \gamma^*)$ 
\end{enumerate}

Moreover, when $\gamma \geq 1/2$, there exists an agent which achieves equality up to a factor of at most $8$ and up to a factor of $16$ if $\epsilon_\rho$ is the relative error.
\end{theorem}


%% file: conclusions.tex
We propose several directions for future research. In general, it would be interesting to explore the central problem of self-modification safety under different agent and environment models and with different assumptions.

\smallskip
\noindent
\textbf{Bounded rationality models.} We analyzed a model of bounded-rationality with a strict upper bound on the size of errors (of several kinds). While this shows that even agents guaranteed to have small errors may self-modify in detrimental ways, the analysis would be significantly different for fully stochastic bounded rationality models (e.g. negligible expected errors with non-negligible variance). One such model of interest is Information-Theoretic Bounded Rationality of \citet{ortega2015informationtheoretic}.

\smallskip
\noindent
\textbf{Awareness of own bounded-rationality.} Whatever underlying decision procedure the agent uses somehow \emph{implicitly} takes its $\epsilon$-optimality into account -- in particular since the assumed $\epsilon$-optimality depends on the behavior of future agent versions. In our formulation, we do not assume the agent to have \emph{explicit} knowledge of its bounded rationality model and $\epsilon$, which would at least intuitively seem useful to know. 

Note, however, that in our framing such explicit knowledge would not be necessarily useful, as any deliberation about it is subject to the same error within $\epsilon$-optimality. Therefore it may be interesting to explore bounded rationality models where the information about own bounded rationality could be explicitly reasoned about (with more precision than e.g. modelling the trajectory of the full environment). Would the agent then be more reluctant to self-modify?

\smallskip
\noindent
\textbf{Time horizons and discounting.} Avoiding temporal discounting would likely yield results with stronger implications (Section~\ref{sec:assumptions}). We propose analysing the finite-time undiscounted case, as well as exploring other means of future value aggregation (finite or infinite, as explored by \citet{Bostrom2011}). 

\smallskip
\noindent
\textbf{Model of self-modification.} As noted above, embedded agents with a strong influence on the environment may self-modify by exploiting the environment. However, the extent of this self-modification, and the strength and stability of mechanisms against self-modification (e.g. via modification-dependent utility function) require further research.

\smallskip
\noindent
\textbf{Probabilistic analysis.} Build stochastic models of agent rationality and self-modification, and perform full probabilistic analysis. This may e.g. inform us about required safety margins. In particular, approaches based on statistical physics and information theory seems to be promising here and have already proven fruitful in analyzing existing optimization problems and algorithms.

%% file: policy2.tex
In this section, we show bounds on how much the expected discounted utility can be changed by self-modifications of $\epsilon$-optimizers.
\begin{repeatthm}{thm:policy_modification}
Let $\rho$ and $u$ be modification-independent. Consider a self-modifying agent which is $\epsilon$-optimizer for the empty history. Then, for every $t \geq 1$, \begin{gather}
E_{\mae_{<t}} [Q\left(\mae_{<t} \pi_{t}\left(\mae_{<t}\right)\right)] \geq E_{\mae_{<t}}[Q\left(\mae_{<t} \pi_{1}\left(\mae_{<t}\right)\right)] - \min(\frac{\epsilon}{\gamma^{t-1}}, \frac{1}{1-\gamma}) \label{eq:1upper_bound_policy_mod}
\end{gather}
where the expectation is with respect to $\mae_{<t}$ such that the perceptions are distributed according to the belief and the actions are given by $a_{i}=\pi_{i}\left(\mae_{<i}\right)$.

Moreover, for all histories $\mae_{<t}$ given by $a_{i}=\pi_{i}\left(\mae_{<i}\right)$ for which the agent is an $\epsilon$-optimizer it holds that
\begin{gather}
Q\left(\mae_{<t} \pi_{1}\left(\mae_{<t}\right)\right) \geq Q\left(\mae_{<t} \pi_{t}\left(\mae_{<t}\right)\right) - \epsilon \label{eq:1lower_bound_policy_mod}\\
\end{gather}
Equality in Inequality \eqref{eq:1upper_bound_policy_mod} can be achieved up to a factor of at most $\gamma$.
\end{repeatthm}
\begin{proof}[Proof:]
\item
\vspace{-15px}
\paragraph{Inequality~\eqref{eq:1upper_bound_policy_mod}} If the future discounted utility lost, in expectation with respect to $\mae_{<t}$, at time $t$ by $\pi_t$ is $\delta$, then the discounted utility lost at time $1$ is at least $\gamma^{t-1} \delta$. Because the agent is an $\epsilon$-optimizer for the empty history, it holds that  $\gamma^{t-1} \delta \leq \epsilon$. We, therefore, have $\delta \leq \frac{\epsilon}{\gamma^{t-1}}$. To see the other branch of the $\min$, note that no more than the maximum achiavable utility $\frac{1}{1-\gamma}$ can be lost at any time step.

\paragraph{Inequality \eqref{eq:1lower_bound_policy_mod}} follows directly from the definition of an $\epsilon$-optimizer -- the $Q$-value at the histories for which the agent is an $\epsilon$-optimizer has to be at most $\epsilon$ lower than the maximum achiavable one and thus also than $Q\left(\mae_{<t} \pi_{1}\left(\mae_{<t}\right)\right)$. This holds because, thanks to modification-independence, the maximum achievable utility does not depend on the specific policy name and its changes.

\paragraph{Tightness of bound~\eqref{eq:1upper_bound_policy_mod}} can be shown by constructing such a utility function, belief and policy. Let us consider the trivial single-perception deterministic belief, set of two actions $\{0,1\}$, utility function which assigns utility $0$ to histories ending with $0$ and utility $1$ to histories ending with $1$. Now consider the family of policies $\{\pi_i\}_{i=1}^\infty$ such that $\pi_i$ always selects $\pi_{i+1}$ as the next policy and performs action $1$ for $i \leq \frac{\ln(1-\gamma)\epsilon}{\ln \gamma} + 1 \eqdef \mathfrak{b}$
and action $0$ for $i > \mathfrak{b}$. It is easy to verify that the discounted utility lost is equal to at most $\gamma^{\mathfrak{b}-1}\frac{1}{1-\gamma} = \epsilon$, meaning that the policy $\pi_1$ is $\epsilon$-optimizing.

It is also easy to check that the discounted utility lost at time $t$ is
\begin{gather}
\min(\gamma^{\ceil{\mathfrak{b}}-t},1)\frac{1}{1-\gamma} \geq \min(\gamma^{\mathfrak{b}+1-t},1)\frac{1}{1-\gamma} = \gamma^{\mathfrak{b}-1} \frac{1}{1-\gamma} \min(\gamma^{2-t},\gamma^{1-\mathfrak{b}}) \geq \\
\geq \epsilon \gamma \min(\gamma^{1-t},\gamma^{-\mathfrak{b}}) = \epsilon \gamma \min(\gamma^{1-t},\frac{\gamma^{-1}}{(1-\gamma)\epsilon}) = \gamma \min(\frac{\epsilon}{\gamma^{t-1}}, \frac{\gamma^{-1}}{1-\gamma}) \geq \gamma \min(\frac{\epsilon}{\gamma^{t-1}}, \frac{1}{1-\gamma})
\end{gather}
which proves the theorem. Moreover, we see that if $\frac{1}{1-\gamma}$ is the smaller branch of $\min()$ in the bound, the bound is tight.
\end{proof}

From this, we now recover Theorem 16 from \cite{everitt}.

\begin{corollary} \label{cor:recover_everitt}
Let $\rho$ and $u$ be modification-independent. Consider a self-modifying $\epsilon$-optimizer. Then, for every $t \geq 1$, perception sequence $e_{<t}$, and the action sequence $a_{<t}$ given by $a_{i}=\pi_{i}\left(\mae_{<i}\right)$, it then holds that
\begin{gather}
Q\left(\mae_{<t} \pi_{t}\left(\mae_{<t}\right)\right)] = Q\left(\mae_{<t} \pi_{1}\left(\mae_{<t}\right)\right)
\end{gather}
almost surely.
\end{corollary}
\begin{proof}
By \Cref{thm:policy_modification}, it follows that 
\begin{gather}
Q\left(\mae_{<t} \pi_{1}\left(\mae_{<t}\right)\right) \geq Q\left(\mae_{<t} \pi_{t}\left(\mae_{<t}\right)\right)
\end{gather}
and that 
\begin{gather}
E_{\mae_{<t}}[Q\left(\mae_{<t} \pi_{1}\left(\mae_{<t}\right)\right)] \leq E_{\mae_{<t}} [Q\left(\mae_{<t} \pi_{t}\left(\mae_{<t}\right)\right)] 
\end{gather}

Since we have that one inequality holds and the other holds in expectation, it must be the case that the two values are equal almost everywhere.
\end{proof}

Note that while the statement in \cite{everitt} does says that the statement holds surely and not merely almost surely, it is easy to come up with a counterexample to that statement, showing that the fact that we have shown that the statement holds almost surely is a fundamental property of the problem and not only artifact of our analysis.

%% file: bounded_eps2.tex
In this section, we show a performance guarantee for perfect optimizers with bounded knowledge.

\begin{remark}
All the proofs work in the setting \textit{without} self-modifications. To show that the results also hold in the case \textit{with} self-modifications, we use the fact that with perfect optimization, self-modifications have no effect on the expected discounted utility gained (as shown in \cite{everitt} or equivalently as implied by \Cref{thm:policy_modification} with $\epsilon = 0$). All the proofs in this section use this implicitly.
\end{remark}

We use the following lemma.

\begin{repeatthm}{lem:optimization_bound}
Let A be a (possibly self-modifying) perfect expected utility maximizer with knowledge $\kappa = (u,\rho, \gamma)$. Let $\kappa^* = (u^*,\rho^*, \gamma^*)$ be the correct knowledge. Assume that
\begin{equation} \label{eq:bound_assumption}
|V^\pi_\kappa(\mae_{<t}) - V^\pi_{\kappa^*}(\mae_{<t})| \leq \epsilon
\end{equation}
for all policies $\pi$ and histories $\mae_{<t}$. Then, A is a $2\epsilon$-optimizer with respect to $\kappa^*$.
\end{repeatthm}
\begin{proof}
Let us define $\pi_{\kappa}^t$ as the policy at time $t$ when starting with some fixed perfect utility optimizer with respect to $\kappa$. Define $\pi_{\kappa^*}^t$ analogously. Note that when $A$ is not self-modifying, it holds that $\pi_{\kappa}^t = \pi_{\kappa}$ and $\pi_{\kappa^*}^t = \pi_{\kappa^*}$ for all $t$.

In the following, the first and the last inequality follow from the assumption. The middle inequality holds because $\pi_\kappa$ is a perfect utility optimizer with respect to knowledge $\kappa$ and $\pi_\kappa^t$ therefore achieves maximum possible expected utility for any history $\mae_{<t}$.
\[
V^{\pi_{\kappa^*}}_{\kappa^*}(\mae_{<t}) -\epsilon \leq V^{\pi_{\kappa^*}}_{\kappa}(\mae_{<t}) \leq V^{\pi_{\kappa}}_{\kappa}(\mae_{<t}) \leq V^{\pi_{\kappa}}_{\kappa^*}(\mae_{<t}) + \epsilon
\]
which we can rearrange to.
\[
V^{\pi_{\kappa}}_{\kappa^*}(\mae_{<t}) \geq V^{\pi_{\kappa^*}}_{\kappa^*}(\mae_{<t}) - 2\epsilon
\]
\end{proof}

%% file: utility2.tex
In this section, we prove the following theorem, giving bounds on optimization of agents with inaccurate utility function.
\begin{repeatthm}{thm:inaccurate_utility}
Let $A$ be a perfect utility maximizer with utility function $u$ and error $\epsilon$ with respect to the correct utility function $u^*$. Then it is a $\frac{2\epsilon}{1-\gamma}$-optimizer with respect to $u^*$. Moreover, this bound is tight.
\end{repeatthm}
Before proving this, we will need the following lemma:
\begin{lemma} \label{lem:utility_bound}
Let $u$ be a utility function with absolute error $\epsilon$ with respect to $u^*$. Then for any fixed policy $\pi$ and any history $\mae_{<t}$.

\[
|V_u^\pi(\mae_{<t}) - V_{u^*}^\pi(\mae_{<t})| \leq \frac{\epsilon}{1-\gamma}
\]
\end{lemma}
\begin{proof}
\begin{gather}
|V_u^\pi(\mae_{<t}) - V_{u^*}^\pi(\mae_{<t})| = |\mathbb{E}(\sum_{t=1}^\infty \gamma^{t-1} u(\mae_{<t})) - \mathbb{E}(\sum_{t=1}^\infty \gamma^{t-1} u^*(\mae_{<t}))| \leq \\
\leq \sum_{t=1}^\infty \gamma^{t-1} |\mathbb{E} u(\mae_{<t}) - \mathbb{E} u^*(\mae_{<t})| \leq \sum_{t=1}^\infty \gamma^{t-1} \epsilon = \frac{\epsilon}{1-\gamma}
\end{gather}
where the linearity of expectation in infinite sums in this special case holds by the monotone convergence theorem.

\end{proof}

\begin{proof}[Proof of \Cref{thm:inaccurate_utility}]
The first part follows from \Cref{lem:utility_bound} together with \Cref{lem:optimization_bound}.

To prove tightness, we construct such a utility function and belief. Let the set of perceptions only has one element (which is given probability $1$ by the belief) and the action set is $\{0,1\}$. The utility $u^*$ is $1$ if the last action was $1$ and $1-2\epsilon$ otherwise. Utility $u$ is always $1-\epsilon$ independently of past actions. We can now assume that the agent always takes action $0$. Then it attains expected discounted utility of $\frac{1-2\epsilon}{1-\gamma}$ instead of the optimal $\frac{1}{1-\gamma}$, making the lost utility $\frac{2\epsilon}{1-\gamma}$.
\end{proof}

%% file: belief2.tex
In this section, we prove \Cref{thm:inaccurate_belief}, giving bounds on optimization of agents with inaccurate belief. The upper bound is in terms of the (weaker) absolute error and lower bounds are in terms of both absolute and relative error.

We use in this section the notion of total variational distance and coupling of random variables. For discussion of these topics, see \cite[Chapter 4.2]{levinmarkov}

\begin{repeatthm}{thm:inaccurate_belief}
Let $A$ be a perfect expected utility maximizer whose belief $\rho$ has absolute error $\epsilon$ with respect to the correct belief $\rho^*$. Then it is a $(\frac{2}{1-\gamma} - \frac{2}{1- \gamma(1-\epsilon)})$-optimizer with respect to $\rho^*$ and this bound is tight up to a factor of $2$. Moreover, if $\epsilon$ is the relative error, this bound is tight up to a factor of 4.
\end{repeatthm}

Before proving this, we will need the following lemma:
\begin{lemma} \label{lem:epsilon_bound_belief}
Let us consider the discounted expected utility with discount factor $\gamma$. Let $\rho$ have absolute error $\epsilon$ with respect to $\rho^*$. Then for any fixed policy $\pi$ and history $\mae_{<t}$.

\begin{gather}
|V_\rho^\pi(\mae_{<t}) - V_{\rho^*}^\pi(\mae_{<t})| \leq \frac{1}{1-\gamma} - \frac{1}{1- \gamma(1-\epsilon)}
\end{gather}

\end{lemma}
Before proving this, we will in turn need the \Cref{lem:belief_v_bound}; its proof is deferred to below.
\begin{lemma} \label{lem:belief_v_bound}
Let $\rho$ have absolute error $\epsilon$ with respect to $\rho^*$. Let $\mathcal{L}_\rho$ be the distribution on $(\mathcal{A}\times \mathcal{E})^*$ induced by $\rho$ and let
\[
\epsilon_{t} = \|\mathcal{L}_\rho (\mae_{k+t} | \mae_k) - \mathcal{L}_{\rho^*} (\mae_{k+t} | \mae_k) \|_{TV} = \sup_{\substack{S \subseteq (\mathcal{A \times E})^{t+k} \\ \text{s.t. first $k$ steps are}\\\text{equal to $\mae_{<k}$}}} |P_\rho(S | \mae_{<k}) - P_{\rho^*}(S | \mae_{<k})|
\]

Then it holds that
\[
\epsilon_{t} = 1-(1-\epsilon)^t
\]
\end{lemma}

\begin{proof}[Proof of \Cref{lem:epsilon_bound_belief}]
Let us bound
\begin{gather}
|V_{\rho}^{\pi}(\mae_{<k}) - V_{\rho^*}^\pi(\mae_{<k})| = |\mathbb{E}_{\rho}(\sum_{t=1}^\infty \gamma^{t-1} u(\mae_{<t+k}) | \mae_{<k}) - \mathbb{E}_{\rho^*}(\sum_{t=1}^\infty \gamma^{t-1} u(\mae_{<t+k}) | \mae_{<k})| \leq \\
\leq \sum_{t=1}^\infty \gamma^{t-1} |\mathbb{E}_\rho (u(\mae_{<t+k}) | \mae_{<k}) - \mathbb{E}_{\rho^*} (u(\mae_{<t+k}) | \mae_{<k})| \stackrel{(1)}{\leq} \sum_{t=1}^\infty \gamma^{t-1} \epsilon_t^\text{abs}  \leq \sum_{t=1}^\infty \gamma^{t-1}(1-(1-\epsilon)^t) = \\
= \frac{1}{\gamma}(\sum_{t=1}^\infty \gamma^{t} - \sum_{t=1}^\infty (\gamma(1-\epsilon))^t) = \frac{1}{1-\gamma} - \frac{1}{1- \gamma(1-\epsilon)}
\end{gather}

where, again, the linearity of expectation in infinite sums holds by the monotone convergence theorem and it, therefore, remains to argue inequality (1). By \Cref{lem:belief_v_bound}, it holds that the total variational distance is at most $\epsilon_t$ and, therefore, there exists a coupling of $X \sim \rho$ and $Y \sim \rho^*$ such that $P(X \neq Y) \leq \epsilon_t$. We can, therefore write for any function $f: (\mathcal{A \times E})^* \rightarrow [0,1]$
\begin{align}
E_\rho f &= E f(X) = E(f(X) | X = Y)P(X = Y) + E(f(X) | X \neq Y)P(X \neq Y) \\
E_{\rho^*} f &= E f(Y) = E(f(Y) | X = Y)P(X = Y) + E(f(Y) | X \neq Y)P(X \neq Y)
\end{align}
which we can now use to bound
\begin{align}
|E_\rho f - E_{\rho^*} f| =& |E(f(X) | X = Y)P(X = Y) + E(f(X) | X \neq Y)P(X \neq Y) \\&- E(f(Y) | X = Y)P(X = Y) + E(f(Y) | X \neq Y)P(X \neq Y)|\\
=& |E(f(X) | X \neq Y)P(X \neq Y) - E(f(Y) | X \neq Y)P(X \neq Y)|\\
\leq& P(X \neq Y) \leq \epsilon_t
\end{align}
where the first of the two inequalities holds because $f(X), f(Y) \in [0,1]$.


\end{proof}

\begin{proof}[Proof of \Cref{lem:belief_v_bound}]
Given a history $\mae_{<k+i}$, define the set $S_{k+i+1}(\mae{k+i})$ as the set of histories of length $k+i+1$ that are equal to $\mae_{<k+i}$ on the first $k+i$ timesteps. Now consider the probability distribution $\rho_{k+i+1, \mae_{<k+i}}$ as the distribution induced by $\rho$ on $S_{k+i+1}(\mae_{k+i})$ by conditioning on the first $k+i$ timesteps being equal to $\mae_{k+i}$ and similarly defined $\rho^*_{k+i+1, \mae_{k+i}}$.

By the \Cref{def:inaccurate_belief_abs}, we have for every $i \leq t$ that $\|\rho_{k+i+1, \mae_{<k+i}} - \rho^*_{k+i+1, \mae_{k+i}}\|_{TV} \leq \epsilon$. There, therefore, exist couplings $X_i \sim \rho_{k+i+1, \mae_{<k+i}}, Y_i \sim \rho^*_{k+i+1, \mae_{<k+i}}$ such that $P(X_i \neq Y_i)$. Let us define $X = (\mae_{<k}, X_1, \cdots X_t)$ and $Y = (\mae_{<k}, Y_1, \cdots, Y_t)$. These two sequences have the same distribution as $\mae_{<k+t}$ with respect to $\rho$ and $\rho^*$, respectively. Moreover, it holds that $P(X \neq Y) \leq 1- (1- \epsilon)^t$. We, therefore, have a coupling of $\mae_{<k+t}$ with respect to the two probbility measures $\rho, \rho^*$ and it follows that their total variational distance is at most $P(X \neq Y) \leq 1- (1- \epsilon)^t$.

\end{proof}

\begin{proof}[Proof of \Cref{thm:inaccurate_belief}]
The first part follows from \Cref{lem:epsilon_bound_belief} together with \Cref{lem:optimization_bound}.

We show tightness by constructing such utility, and belief. Both the action and perception sets are $\{0,1\}$. The utility is $1$ is all past percepts were $1$ and $0$ otherwise. The belief $\rho^*$ is such that when the last action performed was $1$, the perception is $1$ with probability one. If it was $0$, the perception is $1$ with probability $p_1 = 1-2\epsilon$ and $0$ otherwise. In belief $\rho$, independently of the action, the probability of $1$ is $p_2 = 1-\epsilon$. We can now assume that action $0$ is always taken. Then the discounted expected utility at step $k$ is $\gamma^{k-1} (1-2\epsilon)^k$, making the total expected discounted utility $\frac{1}{1- \gamma(1-2\epsilon)}$ instead of the optimal $\frac{2}{1-\gamma}$. The utility lost is then
\begin{gather}
\frac{1}{1-\gamma} - \frac{1}{1- \gamma(1-2\epsilon)} \leq \frac{1}{1-\gamma} - \frac{1}{1- \gamma(1-\epsilon)}
\end{gather}

In the relative case, the construction is the same except that we set $p_1 = \frac{1}{(1+\epsilon)^2}$ and $p_2 = \frac{1}{1+\epsilon}$. The expected discounted utility at step $t$ is then $\gamma^{t-1}\frac{1}{(1+\epsilon)^{2t}}$ resulting in this much utility lost:
\[
\frac{1}{1-\gamma} - \frac{1}{1- \frac{\gamma}{(1+\epsilon)^2}} \leq \frac{1}{1-\gamma} - \frac{1}{1- \frac{\gamma}{(1+\epsilon)}}
\]

Simplifying the ratio of this and $\frac{1}{1-\gamma} - \frac{1}{1- \gamma(1-\epsilon)}$ gives us (we skip the calculations)
\[
\frac{-\gamma +\epsilon +1}{\gamma  (\epsilon -1)+1}
\]
Simplifying the gradient of this, we get (we skip the calculations)
\[
\nabla_{\epsilon, \gamma} \left(\frac{-\gamma +\epsilon +1}{\gamma  (\epsilon -1)+1}\right) = \left(\frac{(\gamma -1)^2}{(\gamma  (\epsilon -1)+1)^2},-\frac{\epsilon ^2}{(\gamma  (\epsilon -1)+1)^2}\right)
\]
The function is, therefore, non-decreasing in $\epsilon$ and non-increasing in $\gamma$ and it has its maximum, subject to $0 \leq \epsilon, \gamma \leq 1$, at $\epsilon=1, \gamma=0$. Evaluating at this point, we get 2 as the value, meaning that the bound is at most a factor of 4 greater than the lower bound in the case of relative error. 
\end{proof}

In the next theorem, we show a lower bound in the average-case scenario which is asymtotically the same as the worst-case one. First we will need the following lemma:

\begin{lemma} \label{lem:probabilistic_bound}
Let $f: \mathbb{R}^n \rightarrow \mathbb{R}$ be a coordinate-wise non-decreasing function and let $X_1, \cdots, X_n$ be random variables. Let us for every $i \in [n]$ have $k_i$ random variables $X_i^1, \cdots, X_i^{k_i}$ such that $L_{X_i^j} \simeq X_i$. Let $f_l = f(X_1^{j_{1,l}}, \cdots, X_n^{j_{n,l}})$ for $l \in [L]$ for some $L$. Consider the set of random variables $\mathcal{F} = \{f_l: l \in [L]\}$ for some set of $j_{i,l}$ s.t. $j_{i,l} \in [k_i]$. Let us consider $f_\text{max} \eqdef \max \mathcal{F}$. Finally, let $l_\text{max}$ be such that $f_\text{max} = f(X_1^{j_{1,l_\text{max}}}, \cdots, X_n^{j_{n,l_\text{max}}})$.

Then $P(X_i^{j_{i,l_\text{max}}} \leq x) \leq P(X_i \leq x)$ for all $i \in [n]$ and $x$.
\end{lemma}
The interpretation is that if we have a non-decreasing function (possibly in multiple variables) and evaluate it several times on possibly different instances of the same random variables, and we take the arguments that resulted in the highest value of the function, then the arguments are generally speaking larger than the value of the respective random variables. More precisely, taking the arguments which resulted in $f_\text{max}$, there exists a coupling of the $i$-th argument with $X_i$ such that the argument is always greater than $X_i$.
\begin{proof}
We have $f_\text{max} \simeq f_1 | f_1 \geq f_i, i \in \{2,3,\cdots,n\}$. We prove that for any value $C$, it holds that $P(X_i^{j_{i,1}} \leq x | f_1 \geq C) \leq P(X_i \leq x)$ and then set $C = \max(f_i, i \in \{2,3,\cdots,n\})$.

Let $p_i(x) = P(X_i=x)$. We consider the set $T$ of $x_1, \cdots, x_n$ such that $f(x_1, \cdots, x_n) \geq C$. Let $S_i(x_i) = \{\mathbf{x} = (x_1, \cdots, x_{i-1}, x_{i}, x_{i+1}, \cdots, x_n), s.t. f(\mathbf{x}) \geq C\}$ and $W_i(x)$ be the hyperplane with the $i$-th component equal to $x$. In other words, $S_i(x) = T \cap W_i(x)$. Note that $S_i(x)$ is non-decreasing. We have that the probabilities
\begin{gather}
p_i(x_i) = \sum_{x\in W_i(x)} \prod_{j=1}^n p_j(x) \\
p_i(x_i | f(x_1, \cdots, x_{i-1}, x_{i}, x_{i+1}, \cdots, x_n) \geq C) = c\sum_{x\in S_i(x)} \prod_{j=1}^n p_j(x)
\end{gather}
for $c = P(f(x) \geq C)^{-1}$ by the Bayes' theorem. Now, because $S_i(x)$ is non-decreasing, we have that for every $i \in [n]$
\[
r_i(x_i) \eqdef \frac{p(x_i | f \geq C)}{p(x_i)} = \frac{c\sum_{x\in S_i(x)} \prod_{j=1}^n p_j(x)}{\sum_{x\in W_i(x)} \prod_{j=1}^n p_j(x)}
\]
is non-decreasing. Since $r_i(x)$ is non-decreasing there exists $x$ (we allow $x$ to be infinite, even though it can be shown that this is never the case) such that for any $x' \leq x$, it holds that $p_i(x') \geq p_i(x' | f \geq C)$ and $p_i(x') \leq p_i(x' | f \geq C)$ for $x' \geq x$. Assume that $x' \geq x$ as the argument for the other direction is analogous and slightly simpler. We have for $C$ equal to 
\begin{gather}
P(X_i^{j_{i,l_\text{max}}} \leq x) = 1-P(X_i^{j_{i,l_\text{max}}} > x) = 1-\sum_{y:y>x} p_i(y | f \geq C) \leq \\ \leq 1-\sum_{y:y>x} p_i(y) = 1-P(X_i > x) = P(X_i \leq x)
\end{gather}
\end{proof}

\begin{repeatthm}{thm:avg_case_belief}
	Let $A$ be a perfect expected utility maximizer whose belief $\rho$ is such that for any $e_t,\mae_{<t}a_t$, the value $\rho(e_t|\mae_{<t}a_t)$ is chosen randomly from the set $\{\max(0,\rho^*(e_t|\mae_{<t}a_t)-\epsilon), \min(1,\rho^*(e_t|\mae_{<t}a_t)+\epsilon)\}$ in the case of absolute error and $\{\frac{\rho^*(e_t|\mae_{<t}a_t)}{1+\epsilon}, \min(1,(1+\epsilon)\rho^*(e_t|\mae_{<t}a_t))\}$ in the case of relative error.
	
	Then, in expectation, the amount of expected discounted utility lost is respectively
	\begin{gather}
	\frac{1}{1-\gamma} - \frac{1}{1- \gamma(1-\epsilon/8)}\\
	\frac{1}{1-\gamma} - \frac{1}{1- \gamma(1-\epsilon/16)}
	\end{gather}
	
	and this is equal to the upper bound for non-random error up to a factor of $16$ and $32$, respectively.
\end{repeatthm}
\begin{proof}
We use a similar construction to the lower bound in the worst case. Both the action and perception sets are $\{0,1\}$. The utility is $1$ if all past percepts were $1$ and $0$ otherwise. The belief $\rho^*$ is such that when the last action performed was $1$, the perception is $1$ with probability $1$. If the last action was $0$, the perception is $1$ with probability $1-\epsilon$ and $0$ otherwise. We call these the $1$-path and $0$-path (note the asymetry in the definitions).

Assume that we are in some vertex of the tree and that the believed probability of perception 1 after action 0 is 1. The probability that the subtree of action 0 has optimum greater or equal to that of the subtree of action 1 happens by symmetry with probability at least $1/2$. Moreover, if the error of action $0$ is in the positive direction, we choose action $0$ with probability at least $1/2$ giving us probability at least $1/8$ of performing action 0.

Therefore, we need to lowerbound the probability that the error is in the positive direction --- in any fixed vertex, this happens with probability exactly $1/2$ but in our case, the probability is conditioned on the fact that the optimal path goes through the vertex. For this, we use \Cref{lem:probabilistic_bound}. For any fixed path, the expected discounted utility is non-decreasing in the absolute errors. We can, therefore, use the lemma to show that if we choose the optimal path, the probability of any of the errors being in the positive direction is at least that for any fixed action (that is $1/2$). Therefore, in each step, we have probability at least $1/8$ that the agent performs action $0$.

We have $\mathbb{E}((1-\epsilon)^{X_t}) \leq \frac{7}{8} + \frac{1}{8}(1-\epsilon) = 1-\frac{\epsilon}{8}$ and by independence $\mathbb{E}\epsilon_t = \mathbb{E}(1-\prod_{i=1}^t (1-\epsilon)^{X_i}) \geq 1-(1-\frac{1}{8}\epsilon)^t$. By a similar argument as above, this gives the lost expected discounted utility of 

\[
\epsilon'_\text{abs} \eqdef \frac{1}{1-\gamma} - \frac{1}{1- \gamma(1-\epsilon/8)}
\]

Similarly for the relative error, $\mathbb{E}(\frac{1}{1+\epsilon}^{X_t}) \leq \mathbb{E}((1-\frac{1}{2}\epsilon)^{X_t}) \leq 1-\frac{1}{16}\epsilon$ and $\mathbb{E}\epsilon_t = \mathbb{E}(1-\prod_{i=1}^t \frac{1}{1+\epsilon}^{X_i}) \geq 1-(1-\frac{1}{16}\epsilon)^t$. Then the expectation of the expected discounted utility lost is no more than
\[
\epsilon'_\text{rel} \eqdef \frac{1}{1-\gamma} - \frac{1}{1- \gamma(1-\epsilon/16)}
\]

Now we bound the ratio of the upper bound $f_\text{bel}(\epsilon, \gamma)$ and the lower bound $\epsilon'_\text{abs}$ (we skip the simplification)

\[
\frac{f_\text{bel}(\epsilon, \gamma)}{\epsilon'_\text{abs}} = \frac{2 (\gamma  (\epsilon -8)+8)}{\gamma  (\epsilon -1)+1} \leq  \frac{2 (\gamma  (8\epsilon -8)+8)}{\gamma  (\epsilon -1)+1} = 16
\]
Similarly for the relative error, we get an upper bound of $32$ on the ratio.

\end{proof}

%% file: gammas2.tex
In this section, we show a bound on utility lost due to the agent having a larger discount factor than the correct one. Recall that $k = \lceil - \frac{1}{\lg \gamma}\rceil$. We use this here to simplify the bounds

In contrast with the other bounds, in this section, we do not use \Cref{lem:optimization_bound} and the proof directly gives a construction of a utility function and belief for which equality is achieved. 

\begin{repeatthm}{thm:imprecise_discount}
Let $\pi_{\gamma}$ and $\pi_{\gamma^*}$ be perfect expected utility maximizers with respect to discount rates $\gamma$ and $\gamma^*$ respectively for some $\gamma \leq \gamma^*$, either with or without the ability to self-modify.. Let $u$,$\rho$ be their utility function and belief. Then 
\[
\left|V^{\pi_{\gamma^*}}\left(\mae_{<t}\right)-V^{\pi_{\gamma}}\left(\mae_{<t}\right)\right| \leq \frac{{\gamma^*}^{k} + {\gamma^*}^{k-1} - 1}{1-\gamma^*} - {\gamma^*}^{k-1}\frac{\gamma^{k} + \gamma^{k-1} - 1}{\gamma^{k-1}(1-\gamma)} 
\]
\end{repeatthm}
\begin{proof}
The proof is the same for both agents with and without the ability to self-modify as it only compares the utilities gained by the respective agents and disregards what history lead to them.

The proof works by formulating the maximum possible amount of expected utility lost as a solution of an optimization problem which we then solve. The objective function of this optimization problem will be the amount of expected utility lost and the constraints will be chosen such that they exactly match the assumptions of the theorem.

Let $\hat{u_t}$ and $\hat{u}'_t$ be the expected utilities gained at timestep $t$ given the policies $\pi_{\gamma^*}$ and $\pi_{\gamma}$ respectively. We can now define
\begin{minipage}[t]{0.5\textwidth}
\begin{gather}
S_1 = \sum_{t=1}^\infty {\gamma^*}^{t-1} \hat{u}_t \\
S'_1 = \sum_{t=1}^\infty {\gamma^*}^{t-1} \hat{u}'_t \\
\end{gather}
\end{minipage}
\hspace{-3em}
\begin{minipage}[t]{0.5\textwidth}
\begin{gather}
S_2 = \sum_{t=1}^\infty \gamma^{t-1} \hat{u}_t \\
S'_2 = \sum_{t=1}^\infty \gamma^{t-1} \hat{u}'_t
\end{gather}
\end{minipage}

We know that the policy leading to utilities $\{\hat{u}'_i\}_i$ is optimal for $\gamma$. This can be written as $S_2 \leq S_2'$ or in other words
\[
\sum_{t=1}^\infty \gamma^{t-1} (\hat{u}_t - \hat{u}'_t) \leq 0
\]

We need to find the largest possible $\epsilon$ such that $S_1 = S_1' + \epsilon$ can be achieved, or equivalently, there exist $\{\hat{u}_i\}_i, \{\hat{u}'_i\}_i$ such that
\[
\sum_{t=1}^\infty {\gamma^*}^{t-1} (\hat{u}_t - \hat{u}'_t) = \epsilon
\]

Substituting $\Delta \hat{u}_t = \hat{u}_t - \hat{u}'_t$, we get the following optimization problem:
\begin{align}
\text{maximize }&\sum_{t=1}^\infty {\gamma^*}^{t-1} \Delta \hat{u}_t \\
\text{subject to }&\sum_{t=1}^\infty \gamma^{t-1} \Delta \hat{u}_t \leq 0, \\
&\Delta \hat{u}_t \in [-1, 1] \quad ,\forall t 
\end{align}

We now prove several lemmas about the structure of the solution to this optimization problem.
\begin{quotation}
\begin{lemma} \label{lem:structure}
The optimum is unique and there exists $k$ such that $\Delta \hat{u}_t = -1$ for $t < k$ and $\Delta \hat{u}_t = 1$ for $t > k$.
\end{lemma}
\begin{proof}
We show that if the second part holds, the optimum is unique. Suppose that there are two optima with $k_1$ and $k_2$ such that $k_1 < k_2$. However, this would mean that for every $i$, we would have that the first is elementwise less or equal to the second and for some element the inequality is strict. This would contradict optimality. For fixed $k$, the optimum choice of $\Delta \hat{u}_k$ is to set it as large as possible without violating constraints, meaning that the solution is unique. It remains to show existence of $k$ from the statement.

Let us fix $\Delta \hat{u}_t$ for all $t \not \in \{i,j\}$ for some $i < j$. We prove that if now $\Delta \hat{u}_i$ can be decreased and $\Delta \hat{u}_j$ increased without violating the second contraint, then the solution is not optimal. This means that in the optimum $\Delta \hat{u}_t$ have to be non-decreasing and at most one can be strictly between $-1$ and $1$, together implying the result.

We will be changing $\Delta \hat{u}_j$ and we find $\Delta_i$ such that the left-hand side of first constraint stays constant or, as the rest in fixed, $\gamma^{i-1} \Delta \hat{u}_i + \gamma^{j-1} \Delta \hat{u}_j$ stays constant. If we increase $\Delta \hat{u}_j$ by $\epsilon$, we have to decrease $\Delta \hat{u}_i$ by $-\gamma^{j-i} \epsilon$. The objective function then changes by $\epsilon {\gamma^*}^{j-1} - \epsilon \gamma^{j-i} {\gamma^*}^{i-1}$. We have
\begin{gather}
\epsilon {\gamma^*}^{j-1} - \epsilon \gamma^{j-i} {\gamma^*}^{i-1} = \epsilon {\gamma^*}^{i-1} ( {\gamma^*}^{j-i} - \gamma^{j-i} )
\end{gather}
Since we assume that $\gamma < \gamma^*$, this is positive, meaning that the objective function increased. Since all constraints remained satisfied, the solution was not optimal, a contradiction.
\end{proof}

\begin{lemma} \label{lem:tight_inequality}
The inequality (6) is tight in any optimal solution, that is $S_2 = S_2'$.
\end{lemma}
\begin{proof}
Suppose for contradiction that the inequality is sharp. This implies that there exists $i$, such that $\Delta \hat{u}_i < 0$. This means that we can increase the obective function by increasing $\Delta \hat{u}_i$ by an amount sufficiently small to not violate any constraint.
\end{proof}
\end{quotation}

Now that we know the structure of the optimal solution, we can get its value.

From \Cref{lem:tight_inequality}, we get
\[
\sum_{t=1}^\infty \gamma^{t-1} \Delta \hat{u}_t = 0
\]
Then we have
\begin{equation} \label{eq:what_is_delta}
- \gamma^{t-1} \Delta \hat{u}_k = \sum_{t=1}^{k-1} \gamma^{t-1} (-1) + \sum_{t=k+1}^{\infty} \gamma^{t-1} 1 = -\frac{1-\gamma^{t-1}}{1-\gamma} + \frac{\gamma^{t}}{1-\gamma} = \frac{\gamma^{t} + \gamma^{t-1} - 1}{1-\gamma}
\end{equation}

Using this, we can easily solve for $k$. By \cref{lem:structure} there is exactly one solution. The lemmas \ref{lem:structure} and \ref{lem:tight_inequality} then provide complete characterization of the optimal solution which allows us to compute the maximum possible $\epsilon$.

First of all, we find $k$. It is the first integer greater or equal to the solution $t$ to $\Delta \hat{u}_t = -1$.
\[
-1 = -\frac{\gamma^{t} + \gamma^{t-1} - 1}{\gamma^{t-1} (1-\gamma)}
\]
Solving this gives us $k = \lceil-\frac{1}{\lg \gamma}\rceil$.

Setting $\Delta \hat{u}_t = -1$ for $t < k$, $\Delta \hat{u}_t = 1$ for $t > k$ and using the above equality to get $\Delta \hat{u}_k$, we have
\[
\epsilon = \sum_{t=1}^{k-1} {\gamma^*}^{t-1} \Delta \hat{u}_t + \sum_{t=k+1}^{\infty} {\gamma^*}^{t-1} + {\gamma^*}^{k-1} \Delta \hat{u}_k = \frac{{\gamma^*}^{k} + {\gamma^*}^{k-1} - 1}{1-\gamma^*} - {\gamma^*}^{k-1}\frac{\gamma^{k} + \gamma^{k-1} - 1}{\gamma^{k-1}(1-\gamma)}
\]
\end{proof}

%% file: combining2.tex
In this section we derive a bound on the performance of an agent that has an inaccurate utility function, belief, discount factor, and is not necessarily a perfect optimizer, thus combining the results from the \cref{sec:policy_mod,sec:optimization_bounds}. The proof works by sequentially applying the individual bounds.

\begin{repeatthm}{thm:combining}
Let $A$ be an $\epsilon_o$-optimizer for the empty history with either (1) the ability to self-modify and modification-independent utility function and belief, or (2) without the ability to self-modify and with a possibly modification-dependent utility function and belief. Let $\gamma$ be the agent's discount rate, $\epsilon_u$ the error in utility function with respect to the correct utility function $u^*$ and $\epsilon_\rho$ the absolute error in belief with respect to the correct belief $\rho^*$. Then at timestep $t$:
\begin{enumerate}[(1)]
\item If we let $\epsilon'$ be smallest such that $A$ at time $t$ is an $\epsilon'$-optimizer, than $E_{\mae_{<t}}[\epsilon'] \leq f_\text{opt}(\epsilon_o, \gamma) + f_\text{util}(\epsilon_u, \gamma) + f_\text{bel}(\epsilon_\rho, \gamma) + f_\text{disc}(\gamma, \gamma^*)$ where the expectation is over histories where percept are distributed according to $\rho^*$ and actions are given by the agent's policy. Moreover, if $\epsilon_o = 0$, than $\epsilon' \leq f_\text{opt}(\epsilon_o, \gamma) + f_\text{util}(\epsilon_u, \gamma) + f_\text{bel}(\epsilon_\rho, \gamma) + f_\text{disc}(\gamma, \gamma^*)$ almost surely.
\item $A$ will be an $\epsilon'$-optimizer, with respect to the correct discount rate $\gamma^*$, where $\epsilon' \leq \epsilon_o + f_\text{util}(\epsilon_u, \gamma) + f_\text{bel}(\epsilon_\rho, \gamma) + f_\text{disc}(\gamma, \gamma^*)$ 
\end{enumerate}

Moreover, when $\gamma \geq 1/2$, there exists an agent which achieves equality up to a factor of at most $8$ and up to a factor of $16$ if $\epsilon_\rho$ is a relative error.
\end{repeatthm}
\begin{proof}
\item
\paragraph{Upper bound:} We prove that in both cases $\epsilon' \leq \epsilon_{o,t} + f_\text{util}(\epsilon_u, \gamma) + f_\text{bel}(\epsilon_\rho, \gamma) + f_\text{disc}(\gamma, \gamma^*)$ with $\epsilon_{o,t}$ being the optimization guarantee of the agent at timestep $t$. From \Cref{thm:policy_modification}, we have that for modification-independent belief and policy, if we start with an $\epsilon$-optimizer, then at step $t$ we have $E_{\mae_{<t}}[\epsilon_{o,t}] \leq \min(\frac{\epsilon}{\gamma^{t-1}}, \frac{1}{1-\gamma})$. In the case of non-self-modifying agents, the optimization guarantee stays the same.

Let $A_0, \cdots, A_5$ be agents such that $A_0$ is a perfect optimizer with perfect knowledge, $A_1$ is the same as $A_0$ except it has an inaccurate discount factor $\gamma$, $A_2$ has additionally inaccurate utility function with error $\epsilon_u$, $A_4$ has additionally absolute error of $\epsilon_\rho$ in belief and $A_5$ has additionally $\epsilon_0$-optimizing policy. Then
\[
|V_{A_0}(\mae_{<t}) - V_{A_5}(\mae_{<t}))| \leq \sum_{i=1}^4 |V_{A_i}(\mae_{<t}) - V_{A_{i+1}}(\mae_{<t}))|
\]

In this sum, the summands can be bounded by $f_\text{disc}(\gamma_u, \gamma_o), f_\text{util}(\epsilon_u, \gamma_u), f_\text{bel}(\epsilon_\rho, \gamma_u)$, and $\epsilon_{o,t}$ respectively thanks to \cref{thm:inaccurate_utility,thm:inaccurate_belief,thm:imprecise_discount}.

The bound in the case (2) follows directly. In the case of (1), this gives us
\[
E_{\mae_{<t}}[\epsilon'] \leq E_{\mae_{<t}}[\epsilon_{o,t} + f_\text{util}(\epsilon_u, \gamma) + f_\text{bel}(\epsilon_\rho, \gamma) + f_\text{disc}(\gamma, \gamma^*)] = f_\text{opt}(\epsilon_o, \gamma) + f_\text{util}(\epsilon_u, \gamma) + f_\text{bel}(\epsilon_\rho, \gamma) + f_\text{disc}(\gamma, \gamma^*)
\]

\paragraph{Lower bound:} We show the bound in the case of absolute error in belief, the proof for the relative error is analogous. One of $f_\text{opt}, f_\text{util}(\epsilon_u, \gamma_u), f_\text{bel}(\epsilon_\rho, \gamma_u), f_\text{disc}(\gamma_u, \gamma_o)$ has to be greater or equal to $\epsilon'/4$. Let this be the case for $f_\bullet$ for $\bullet \in \{\text{opt, util, bel, disc}\}$. We know from theorems \ref{thm:inaccurate_utility}, \ref{thm:inaccurate_belief} and \ref{thm:imprecise_discount} that there exists an agent such that $\epsilon'_\bullet$ is tight up to a factor of $2$ (for bounded-optimization agent, we are using that $\gamma \geq 1/2$). For this agent, $\epsilon'$ is then tight up to a factor of $8$.
\end{proof}

%% file: paper.bbl
\begin{thebibliography}{14}
\providecommand{\natexlab}[1]{#1}
\providecommand{\url}[1]{\texttt{#1}}
\expandafter\ifx\csname urlstyle\endcsname\relax
  \providecommand{\doi}[1]{doi: #1}\else
  \providecommand{\doi}{doi: \begingroup \urlstyle{rm}\Url}\fi

\bibitem[Armstrong et~al.(2012)Armstrong, Sandberg, and Bostrom]{OracleAI}
Stuart Armstrong, Anders Sandberg, and Nick Bostrom.
\newblock Thinking inside the box: Controlling and using an oracle {AI}.
\newblock \emph{Minds and Machines}, 22, 11 2012.
\newblock \doi{10.1007/s11023-012-9282-2}.

\bibitem[Bostrom(2011)]{Bostrom2011}
Nick Bostrom.
\newblock {Infinite Ethics}.
\newblock \emph{Analysis and Metaphysics}, 10:\penalty0 9--59, 2011.
\newblock URL \url{www.nickbostrom.com/ethics/infinite.pdf}.

\bibitem[Bostrom(2014)]{Superintelligence}
Nick Bostrom.
\newblock \emph{Superintelligence: Paths, Dangers, Strategies}.
\newblock Oxford University Press, Inc., USA, 1st edition, 2014.
\newblock ISBN 0199678111.

\bibitem[Demski and Garrabrant(2019)]{EmbeddedAgency}
Abram Demski and Scott Garrabrant.
\newblock Embedded agency.
\newblock \emph{CoRR}, abs/1902.09469, 2019.
\newblock URL \url{http://arxiv.org/abs/1902.09469}.

\bibitem[Drexler(2019)]{drexler2019reframing}
K.~Eric Drexler.
\newblock Reframing superintelligence: Comprehensive {AI} services as general
  intelligence.
\newblock \emph{Future of Humanity Institute, University of Oxford}, 2019.

\bibitem[Everitt et~al.(2016)Everitt, Filan, Daswani, and Hutter]{everitt}
Tom Everitt, Daniel Filan, Mayank Daswani, and Marcus Hutter.
\newblock Self-modification of policy and utility function in rational agents.
\newblock \emph{CoRR}, abs/1605.03142, 2016.
\newblock URL \url{http://arxiv.org/abs/1605.03142}.

\bibitem[Lattimore and Hutter(2014)]{Lattimore2014}
Tor Lattimore and Marcus Hutter.
\newblock General time consistent discounting.
\newblock \emph{Theoretical Computer Science}, 519:\penalty0 140 -- 154, 2014.
\newblock ISSN 0304-3975.
\newblock \doi{https://doi.org/10.1016/j.tcs.2013.09.022}.
\newblock URL
  \url{http://www.sciencedirect.com/science/article/pii/S0304397513007135}.
\newblock Algorithmic Learning Theory.

\bibitem[Levin et~al.(2009)Levin, Peres, and Wilmer]{levinmarkov}
D.A. Levin, Y.~Peres, and E.L. Wilmer.
\newblock \emph{Markov Chains and Mixing Times}.
\newblock American Mathematical Soc., 2009.
\newblock ISBN 9780821886274.
\newblock URL \url{https://books.google.dk/books?id=6Cg5Nq5sSv4C}.

\bibitem[Miller et~al.(2020)Miller, Yampolskiy, and
  H{\"a}ggstr{\"o}m]{miller2020agi}
James~D Miller, Roman Yampolskiy, and Olle H{\"a}ggstr{\"o}m.
\newblock An {AGI} modifying its utility function in violation of the
  orthogonality thesis.
\newblock \emph{arXiv preprint arXiv:2003.00812}, 2020.

\bibitem[Omohundro(2008)]{omohundro2008basic}
Stephen~M Omohundro.
\newblock The basic {AI} drives.
\newblock In \emph{AGI}, volume 171, pages 483--492, 2008.

\bibitem[Ortega et~al.(2015)Ortega, Braun, Dyer, Kim, and
  Tishby]{ortega2015informationtheoretic}
Pedro~A. Ortega, Daniel~A. Braun, Justin Dyer, Kee-Eung Kim, and Naftali
  Tishby.
\newblock Information-theoretic bounded rationality, 2015.

\bibitem[Russell(2019)]{russell2019human}
Stuart Russell.
\newblock \emph{Human compatible: artificial intelligence and the problem of
  control}.
\newblock Viking, New York, New York, 2019.
\newblock ISBN 9780525558613.

\bibitem[Strotz(1955)]{strotz1955myopia}
Robert~Henry Strotz.
\newblock Myopia and inconsistency in dynamic utility maximization.
\newblock \emph{The review of economic studies}, 23\penalty0 (3):\penalty0
  165--180, 1955.

\bibitem[Yampolskiy(2012)]{Leakproofing}
Roman Yampolskiy.
\newblock Leakproofing the singularity artificial intelligence confinement
  problem.
\newblock \emph{Journal of Consciousness Studies}, 19:\penalty0 194--214, 01
  2012.

\end{thebibliography}
